\documentclass[letterpaper, 10 pt, conference]{ieeeconf}  

\IEEEoverridecommandlockouts                              
\usepackage{graphics} 
\usepackage{epsfig} 
\usepackage{mathptmx} 
\usepackage{times} 
\usepackage{amsmath} 
\usepackage{amssymb}  

\usepackage{subfig}
\usepackage{graphicx}

\title{\LARGE \bf
Progressive Explanation Generation for Human-robot Teaming
}

\author{Yu Zhang and Mehrdad Zakershahrak   
\thanks{Yu Zhang and Mehrdad Zakershahrak are with the Computer Science and Engineering
Department at Arizona State University, {\tt\small \{yzhan442\}@asu.edu, mehrdad@asu.edu}.}%
}

\usepackage{amsfonts}
\usepackage{booktabs}
\usepackage{siunitx}
\usepackage{times}
\usepackage{helvet}
\usepackage{courier}
\usepackage{amsmath}
\usepackage{algorithm}
\usepackage{algorithmic}
\usepackage{multirow}
\usepackage{graphics}
\usepackage{xcolor}
\usepackage[normalem]{ulem}

\newtheorem{definition}{Definition}

\newtheorem{theorem}{Theorem}
\newtheorem{problem}{Problem}

\DeclareMathOperator*{\argmin}{arg\,min}

\begin{document}
%
\maketitle

\begin{abstract}
Generating explanation to explain its behavior is an essential capability for a robotic teammate.
Explanations help human partners better understand the situation and maintain trust of their teammates.
Prior work on robot generating explanations
focuses on providing the reasoning behind its decision making. 
These approaches, however, fail to heed the cognitive requirement of understanding an explanation. 
In other words, while they provide the right explanations from the explainer's perspective,
the explainee part of the equation is ignored. 
In this work, we address an important aspect along this direction that contributes to a better understanding of a given explanation,
which we refer to as the {\it progressiveness} of explanations. 
A progressive explanation improves understanding by limiting the cognitive effort 
required at each step of making the explanation.
As a result, such explanations are expected to be ``smoother'' and hence easier to understand.
A general formulation of progressive explanation is presented. 
Algorithms are provided based on several alternative quantifications of cognitive effort
as an explanation is being made, which are evaluated in a standard planning competition domain. 
\end{abstract}

\section{Introduction}
\label{sec:intro}

Similar to teaming between humans, a robotic teammate must often explain its behavior to its human partners.
Explanations in such a teaming context provide the reasoning behind one's decision making \cite{lombrozo2006structure},
and help with building a shared situation awareness and maintaining trust between teammates \cite{154193128803200221, nancy2015}.
Although there exists prior work on generating explanations, 
those approaches often ignore the cognitive requirement of understanding an explanation.
The focus there is on generating the right explanations from the explainer's perspective rather than good explanations for the {\it explainee} \cite{sohrabi2011preferred, excuses-icaps, hanheide2017robot}. 
Unsurprisingly, the right explanation may not necessarily be a good explanation--anyone
with parental experience would share the sympathy. 
Such dissonance may be due to many reasons, 
such as information asymmetry and different cognitive capabilities, to name a few.

We summarize such discrepancies as {\it model differences}--the differences 
between the cognitive models that govern the generation and interpretation of an explanation. 
This follows our intuition since that, 
assuming the explainer is incentivized to make the explainee understand 
the decision in the same way as he does\footnote{An arguably more interesting situation may be for the explainer to deliberately introduce knowledge gap between the two parties under secretive intentions.}, 
an explanation from the perspective of the explainer must be perfectly right and understandable,
that is, if the explanation were to made to the explainer himself. 
Unfortunately, the purpose of an explanation is for the explainee,
who may have a very different model for interpreting the explanation.
In our prior work, we investigated how explanations can be made subject to such model differences \cite{chakraborti2017plan},
where the focus is on generating explanations that also make sense given the model of the explainee.

In this work, we take a step further by generating explanations while also 
considering the differences between the cognitive capabilities that may be present between the explainer and explainee.
This is especially relevant to human-robot teaming since robots
are frequently being deployed to situations that require high cognitive (computational) powers
that humans do not have.  
Thus, the motivation here is to generate explanations that minimize the cognitive effort required for understanding them. 
We first note that making an explanation is not an instantaneous task; information must be 
conveyed in a sequential order. 
Hence, this key to reducing cognitive effort is equivalent to minimizing the sum of the effort
required at each step as an explanation is being made. 
This means that not only the collection of information but also
the sequence of presenting it matters. 
Consequently, we term our approach {\it progressive explanation generation} to capture that aspect. 
Consider the following example of a conversation between two friends,
which illustrates the importance of providing information in the right order 
when making an explanation.

\vskip3pt
\noindent
{\tt
Amy: Let's go to the outlet today.\\
Monica: My car is ready. \\
Amy: Great! \\
Monica: The rain will stop soon. \\
Amy: Wonderful! \\
Monica: By the way, today is a holiday (shops closed). \\
Amy: You are telling me now!\\
Monica: Let us go to the central park!\\
Amy: ...
}
\vskip3pt

In this paper, we provide a general formulation of progressive explanation generation 
to avoid such issues in the above example that frequently occur in our lives.
We provide search methods with several alternative quantifications of cognitive effort 
at each step as an explanation is being made. 
Here, we focus on the problem formulation and efficient solution methods. 
Evaluation with human subjects to validate the usefulness of such explanations is delayed to future work. 
Next, we first review related work, followed by
a brief discussion of our prior work on explanation generation.
The formulation of progressive explanation and evaluations are provided afterwards.

\section{Related Work}
\label{related}

Explainable AI \cite{gunning2017explainable} is increasingly considered an important paradigm for designing future intelligent agents,
especially as such systems start to constitute an important part of our lives. 
The key requirement of explainable agency \cite{langley2017explainable} is to be ``explainable'' to the human partners.
To be explainable, an agent must not only provide a solution to a given problem,
but also make sure that the solution is perceived as such. 
A determining factor here is the human interpretation of the agent's behavior. 
It is not difficult to think of situations where the agent's help would be interpreted as no more than an interruption, 
which resulted in the pitfall of earlier effort in designing intelligent assistants,
such as the loss of situation awareness and trust \cite{endsley2016designing, langfred2004too}.   

The key challenge to explainable agency hence is the ability to model the human cognitive model
that is responsible for interpreting the behavior of other agents \cite{chakraborti2017ai}.
With such a model, there are different ways to make the robot's behavior explainable. 
One way is to modulate the robot's behavior towards the human's expectation of it based on the human cognitive model. 
Under this framework, a robot may generate legible motions \cite{Dragan-RSS-13} or explicable plans \cite{zhang2017plan}. 
Essentially, the robot sacrifices the plan quality to respect the human's expectation--the resulting plan is often a more costly plan.  
Another way is to signal its intention before execution \cite{gong2018robot}.
The intuition there is to provide additional context that helps explain the robot's decision.   
The most relevant way to this work is for the robot to explain its decision via explanation generation \cite{excuses-icaps, hanheide2017robot, sohrabi2011preferred}. 
The benefit of generating explanations, compared to generating explainable plans, 
is that the robot can keep its original plan. 
However, as mentioned earlier, the focus there is often on providing the right reasoning from the explainer's perspective,
not necessarily the explanation that is good for the explainee. 
Our prior work \cite{chakraborti2017plan} addresses this gap by proposing a new method of generating explanations
as a model reconciliation problem, 
which takes into account the explainee's model. 
Although some formulations of explanations there implicitly consider 
the cognitive requirement of understanding an explanation from the explainee, 
the main focus there is on generating explanations that reconcile the model differences,
so that the robot plan would seem to be right in the reconciled model of the explainee. 

The idea behind generating progressive explanations bears similarities to the idea of nudging or persuasion \cite{petty1981cognitive}. 
For example, in robotic and AI systems, the goal of nudging is to gradually nudge the human towards a new path \cite{lien2004shepherding}
or provide constant and nonintrusive reminders for performing various tasks \cite{maxwell1999alfred}. 
The general idea to develop a ``smooth'' (or socially acceptable \cite{miller2017explanation}) transition, whether physical or cognitive, to the objective.
We implement a similar idea here for explanation generation. 
To minimize the cognitive effort required at individual steps that lead to the objective, 
we make use of several alternative quantifications of cognitive effort that have connection to,
for example, the distance between the plans \cite{fox2006plan} of two adjacent steps as a result of the additional information provided during the explaining process. 
Intuitively, changes lead to cognitive effort. 

\section{Explanation Generation}

Since our problem setting is based on our  prior work on explanation generation \cite{chakraborti2017plan}, 
we first provide a brief review. 
The problem setting is presented in Fig. \ref{setting}
where two models are presented,
one for the robot ($M^R$) and one for the human ($M^H$), respectively.
The human uses $M^H$ to generate his expectation of the robot's behavior, $\pi_{M^H}$.
When it does not match the actual robot's behavior, $\pi_{M^R}$, (often generated by $M^R$ without considering $M^H$),
the robot becomes unexplainable.

\begin{figure}
\centering
  \includegraphics[width=0.7\columnwidth]{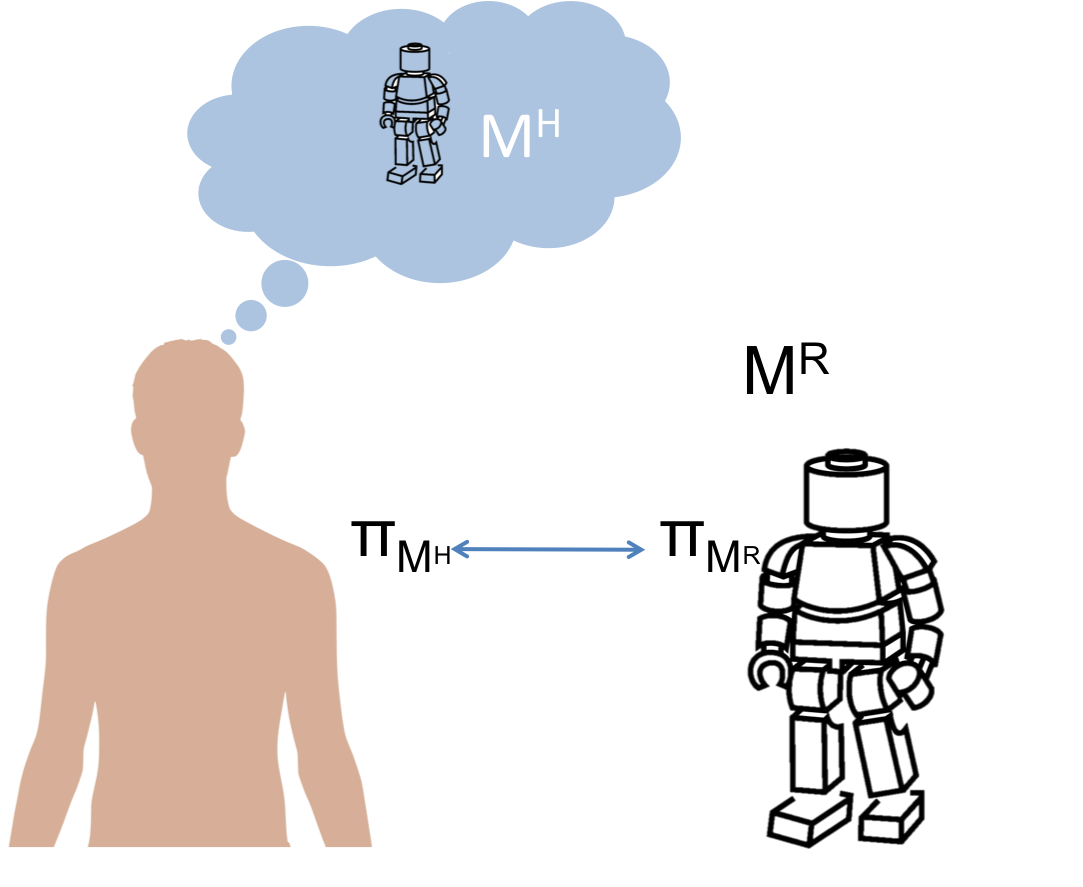}%
\caption{Problem setting of explanation generation. $M^R$ denotes the robot model and $M^H$ denotes the human model that is used to generate his expectation of the robot's behavior ($\pi_{M^H}$). When the expectation does not match the robot's behavior, $\pi_{M^R}$, explanations must be generated.}
 \label{setting}
\end{figure}

\begin{definition}[Model Reconciliation Setting]
The model reconciliation setting is a tuple $(\pi^*_{I, G}, \langle M^R, M^H\rangle)$,
where $cost(\pi^*_{I, G}, M^R) = cost^*_{M^R}(I, G)$.
\end{definition}
where $\pi_{I, G}^*$ is the robot's plan to be explained.
$cost(\pi^*_{I, G}, M^R)$ returns the cost of a plan under the model $M^R$,
and $cost^*_{M^R}(I, G)$ returns the optimal plan for a given initial and goal state pair under $M^R$.
The constraint of $cost(\pi^*_{I, G}, M^R) = cost^*_{M^R}(I, G)$ ensures that
the robot plan is optimal in its own model. 

Explanation generation under this setting is about how to bridge the two models so that
the robot plan $\pi^*_{I, G}$ also becomes explainable (optimal) in the human model after reconciliation as a result of the explanation. 
In this setting, the robot would assume what it has is the correct model and start from there.\footnote{If the robot were to know that the model was incorrect,  
it should have updated it in the first place!}
As a result, explanation for a model reconciliation setting can be considered
as requesting changes to the model of the human. 
Note that making an explanation may also lead to an error report 
if it is found out that the robot model is incorrect.

To capture model changes, a model function $\Gamma: \mathcal{M} \rightarrow S$ is defined to convert a model
to a set of model features.
In such a way, one model can be updated to another model with editing functions
that change one feature at a time. 
The set of feature changes is denoted as $\Delta(M_1, M_2)$
and the distance between two models as the number of such feature changes is denoted as $\delta(M_1, M_2)$.
In our prior and this work, we assume that the model is defined in PDDL \cite{fox2003pddl2}, 
which is similar to the STRIPS model \cite{fikes1971strips} where a model is specified as a tuple $M = (D, I, G)$. 
The domain $D = (F, A)$ is composed of a set of predicate variables, $F$, and a set of actions, $A$. 
$F$ is used to specify the state of the world.
Each action $a \in A$ can be represented as $a = ({\it pre}(a), {\it eff}^+(a), {\it eff}^-(a), c)$,
which denote the preconditions, add and delete effects, and cost of the action, respectively. 
For example, a very simple model for {\tt Amy} in our motivating example would be:

\vskip3pt
\noindent
{\tt
{\textbf{Initial state}}: not-holiday\\
{\textbf{Goal state}}: happy\\
\vskip1pt
\noindent 
{\textbf{Actions}}: \\
OUTLET-SHOPPING 5 (1)\\
pre: not-holiday (car-ready is-sunny)\\
eff$^+$: happy
\vskip3pt
\noindent 
VISIT-PARK 10 (9)\\
pre: (car-ready is-sunny)\\
eff$^+$: happy
}
\vskip3pt
The cost of an action is attached after the action name. 
For simplicity, we use only boolean variables above and take a few shortcuts for the notations. 
For example, for the action {\tt OUTLET-SHOPPING},
the cost is $5$ when neither car is ready or is sunny, 
and $1$ when they are.\footnote{In PDDL, this will be treated as two separate actions.}
The goal is to achieve the effect of {\tt happy} with the minimum cost.
In this example, the model, denoted as $M_{{\tt Amy}}$, will be
converted by the model function to:
\vskip2pt
$\Gamma(M_{{\tt Amy}}) = \{\\\text{{\tt init-has-not-holiday}},\\\text{{\tt goal-has-happy}},\\ \text{{\tt OS-has-precondition-not-holiday}}, \\\text{{\tt OS-has-add-effect-happy}},...\}$\\
where $OS$ is short for {\tt OUTLET-SHOPPING}.
The function essentially turns a model into a set of features that fully specify the model. 
Hence, changing the set of features will also change the model. 

\begin{definition}[Explanation Generation Problem]
The explanation generation problem is a tuple $(\pi^*_{I, G}, \langle M^R, M^H\rangle)$,
and an explanation is a set of unit feature changes to $M^H$ such that 
$1)$ $\Gamma(\widehat{M^H}) \setminus \Gamma(M^H) \subseteq \Gamma({M^R})$, and
$2)$ $cost(\pi^*_{I, G}, {\widehat{M^H}}) - cost_{\widehat{M^H}}^*(I, G) < cost(\pi^*_{I, G}, {M^H}) - cost_{M^H}^*(I, G)$,
where $\widehat{M^H}$ denotes the model after the changes.
\label{def:exp}
\end{definition}
The first condition requires that the changes to the human model must be consistent with the robot model.
This is reasonable given our previous assumption about the motivation of the agent. 
The second condition states that the robot's plan must be closer (in terms of cost) to the optimal plan after the model changes
 than before,
 since otherwise the explanation does not explain the robot's behavior,
 assuming that the human is rational.

\begin{definition}[Complete Explanation]
A complete explanation for an explanation generation problem is an explanation that
in addition satisfies $cost(\pi^*_{I, G}, {\widehat{M^H}}) = cost_{\widehat{M^H}}^*(I, G) = cost(\pi^*_{I, G}, M^R)$.\footnote{
We deviate from our prior work \cite{chakraborti2017plan} a bit requiring that the cost in $\widehat{M^H}$ is equivalent to the cost in $M^R$.
This is important for the human to also associate the right cost to the robot's plan .
}
\end{definition}

A complete explanation requires the model changes to make the robot's plan also optimal in the changed human model
and that the cost is consistent with that in the robot's model.
There may exist multiple complete explanations for a given explanation generation problem.
Some of them may be the most {\it concise} (i.e., contains minimum unit feature changes) and some may be {\it monotonic} (no additional changes that satisfy condition $1$ in Def. \ref{def:exp} may violate the condition of complete explanation). 
A concise explanation can be found using a model-space search over the set of possible model updates that may be made to $M^H$. The search starts from $M^H$ and incrementally adds more changes. A complete explanation searches from $M^R$ and takes aways changes that do not contribute to the violation of the condition of a complete explanation. 
An example of $\widehat{M_{{\tt Amy}}}$ (corresponds to $\widehat{M^H}$) after a complete explanation is:

\vskip3pt
\noindent
{\tt
{\textbf{Initial state}}:  {\sout{not-holiday}} car-ready (+) is-sunny (+)\\
{\textbf{Goal state}}: happy\\
\vskip1pt
\noindent 
{\textbf{Actions}}: \\
OUTLET-SHOPPING 5 (1)\\
pre: not-holiday (car-ready is-sunny)\\
eff$^+$: happy
\vskip3pt
\noindent 
VISIT-PARK 10 (9)\\
pre: (car-ready is-sunny)\\
eff$^+$: happy
}

where the strikeout is the feature removed and the $+$ following a feature denotes an addition.
These changes correspond to the explanation made in our motivating example.   
In this case, the robot model, $M^R$, corresponds to $M_{{\tt Monica}}$,
which is the same as $\widehat{M_{{\tt Amy}}}$ after the explanation (where the model changes incurred).

\section{Progressive Explanation Generation}

Although the {\it conciseness} of explanations \cite{chakraborti2017plan} in our prior work takes into account the amount of information,
which is connected to the cognitive effort required for understanding an explanation,
a concise explanation may not always be the best explanation. 
Our motivating example provides a perfect illustration of when the 
ordering may make a huge difference. 
Furthermore, a good explanation may sometimes involve information with some level of redundancy that helps with understanding \cite{mesmer2009information}. 
Here, we set out to develop a novel process to generate such explanations,
with the focus on reducing the cognitive effort. 

Given the general formulation of explanation in Def. \ref{def:exp},
an explanation is expressed as a set of unit feature changes to $M^H$.
The implication here is that making explanation is an incremental process,
where each step may represent a unit feature change. 
The cognitive effort then can be viewed as the sum of effort 
associated with understanding each change in a sequential order.
We couple the cognitive effort for each change with a general model-plan distance metric,
denoted as 
$\rho(\langle M_1, \pi_1\rangle, \langle M_2, \pi_2 \rangle)$,
where $M_1, \pi_1$ is the model and plan before the change,
and  $M_2, \pi_2$ is that after the change.

\begin{definition}[Progressive Explanation Generation (PEG)]
A progressive explanation is a complete explanation with an ordered 
sequence of unit feature changes that minimize the sum of the
model-plan distance metric:
$\argmin_{{\langle\Delta(\widehat{M^H}, M^H)\rangle}}{\sum_{f_i \in {\langle\Delta(\widehat{M^H}, M^H)\rangle}}{\rho_i}}$,
where $\rho_i$ is short for $\rho_i(\langle M_{i-1}, \pi_{i-1}\rangle, \langle M_i, \pi_i \rangle)$ and $i$ is the index of the model changes starting from $1$, and  $f_i$ denotes the $i$th unit feature change.
\label{peg}
\end{definition}
The angle brackets above convert a set to an ordered set 
and the summation is over the changes required 
for a complete explanation--computed for before and after each unit feature change
is made in a progressive fashion. 
More specifically, $M_0 = M^H, \pi_{0} = \pi^*_{M^H}(I, G)$,
and $M_i = \widehat{M^H}, \pi_{i} = \pi^*_{\widehat{M^H}}(I, G)$
where $\widehat{M_i}$ denotes $M^H$
after model changes of $f_{1:i}$.

\subsection{PEG with Different Distances}

Depending on how the model-plan distance metric is defined, 
different explanation may be resulted. 
Next, we look at a few options for defining this distance,
which intuitively have an impact on cognition. 
Search methods based on these options are provided afterwards. 

\begin{problem}
Progressive explanation generation with 
\begin{equation}
\rho_i = |cost^*_{{M_{i-1}}}(I, G) - cost^*_{{M_{i}}}(I, G)|
\end{equation}
\label{prob1}
\end{problem}

In this case, the distance at each step is characterized by the cost difference of the plans in the two models adjacent to a unit feature change, respectively. 
The search problem is in the model space and is expensive to solve.
Here, we can take advantage of the following equation,
which follows from basic arithmetics:

\begin{equation}
\sum_i \rho_i \geq |cost^*_{M^H}(I, G) - cost^*_{\widehat{M^H}}(I, G)|
\label{equ1}
\end{equation}

The equality above holds if and only if the changes in plan cost are monotonic with respect to the index $i$.
This also reflects the progressive nature of such explanations. 
This observation leads to an efficient heuristic,
where 
\begin{equation}
h(M_i) = |cost^*_{M_{i}}(I, G) - cost(\pi^*_{I, G}, \widehat{M^H})|
\end{equation}
Additionally, 
without the loss of generality, 
assuming that $cost^*_{M^H}(I, G) \leq cost^*_{\widehat{M^H}}(I, G) = cost(\pi^*_{I, G}, \widehat{M^H})$ is satisfied,
the search process could first check 
adding preconditions, removing add effects,
adding delete effects, or increasing action costs.
Since these changes will increase the cost of the plan,
they will more likely lead to faster search process.  

\begin{theorem}
The heuristic described above is admissible and consistent for problem \ref{prob1}. 
\label{thm1}
\end{theorem}

\begin{proof}
It can be easily verified that $h(M_i) = |cost^*_{M_{i}}(I, G) - cost(\pi^*_{I, G}, \widehat{M^H})| \leq \sum_{k > i} \rho_k$.
Hence, the heuristic above is admissible. 
For consistency, notice that $ |cost^*_{M_{i-1}}(I, G) - cost(\pi^*_{I, G}, \widehat{M^H})| = h(M_{i-1})  \leq \rho_i + h(M_{i}) = |cost^*_{{M_{i-1}}}(I, G) - cost^*_{{M_{i}}}(I, G)| + |cost^*_{M_{i}}(I, G) - cost(\pi^*_{I, G}, \widehat{M^H})|$.
\end{proof}

\begin{problem}
Progressive explanation generation with 
\begin{equation}
\rho_i = |cost^*_{M_{i-1}}(I, G) - cost^*_{M_{i}}(I, G)|^2
\end{equation}
\label{prob2}
\end{problem}

Compared to the first problem, the distance metric 
requires that the cost gaps across adjacent steps are as equally distributed as possible. 
For this problem, 
one possible heuristic is:
\begin{equation}
h(M_i) = 0.5 * |cost^*_{M_{i}}(I, G) - cost(\pi^*_{I, G}, \widehat{M^H})|^2
\end{equation}

\begin{theorem}
The heuristic described above is admissible and consistent for problem \ref{prob2}. 
\label{thm2}
\end{theorem}

\begin{proof}
This proof is similar to the proof for problem \ref{prob1} following the inequality that $a_1^2 + a_2^2 + ... + a_n^2 \leq 0.5 * (a_1 + a_2 + ... + a_n)^2$.
\end{proof}


%
%

\begin{figure}
\centering
  \includegraphics[width=0.6\columnwidth]{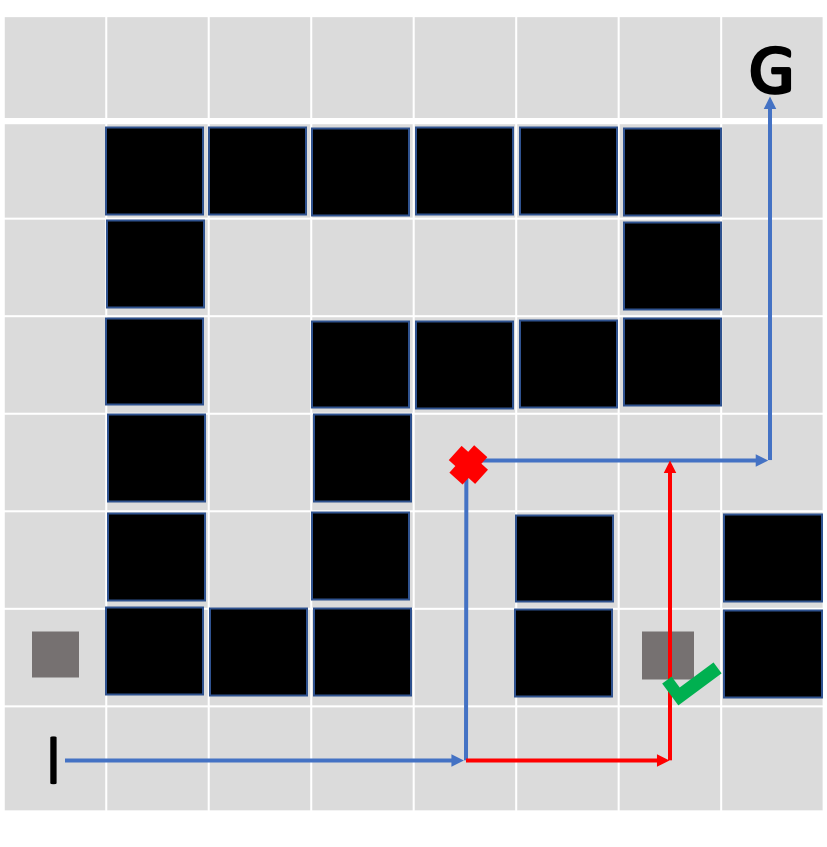}%
\caption{A path-finding scenario that illustrates progressive explanation with plan editing distance.
The initial and goal positions are marked by $I$ and $G$, respectively. 
Black squares are obstacles. 
In this scenario, the original plan is no longer feasible and there are two alternatives that are originally known to be blocked but actually clear. 
The squares that are originally believed to be blocked are marked as small gray squares. 
A progressive explanation with plan editing distance would consider explaining the clearing of the obstacle marked by the green tick, leading to the changes to the original path (which is blocked) shown in red. 
}
 \label{distance}
\end{figure}

Many times, it is not the cost of the plan that matters. 
Instead, we need to look at how significant the changes to plans are across different steps. 
When the change to the plan is less significant before and after a change, 
we only need to slightly modify the existing plan to reach the new plan.
This leads to less cognitive effort.  
An illustration of this through an example is presented in Fig. \ref{distance}.

%
%
%

\begin{problem}
Progressive explanation generation with 
\begin{equation}
\rho_i = d(\pi_{M_{i-1}}^*(I, G), \pi_{M_{i}}^*(I, G))
\end{equation}
\label{prob3}
\end{problem}
where $d(\pi_{M_{i-1}}^*(I, G), \pi_{M_{i}}^*(I, G))$ denotes the minimum editing distance between the two optimal plans
that are created in $M_{i-1}$ and $M_i$, respectively. 
Note that $\pi_i = \pi^*_{I, G}$, which is the robot plan to be explained. 
Similarly, we can apply the following heuristic:
\begin{equation}
h(M_i) = d(\pi^*_{M_i}(I, G), \pi_{I, G}^*)
\end{equation}

\begin{theorem}
The heuristic described above is admissible and consistent for problem \ref{prob3}. 
\label{thm3}
\end{theorem}

\begin{proof}
The plan editing distance clearly satisfies $\sum_i d_i \geq d(\pi_{M^H}^*(I, G), \pi^*_{I, G})$,
since the distance metric is positive and symmetric. 
A similar proof follows from Theorem \ref{thm1}.
\end{proof}


\begin{problem}
Progressive explanation generation with 
\begin{equation}
\rho_i = d^2(\pi_{M_{i-1}}^*(I, G), \pi_{M_{i}}^*(I, G))
\end{equation}
\label{prob4}
\end{problem}
Similarly, we can use the following heuristic:
\begin{equation}
h(M_i) = 0.5 * d^2(\pi^*_{M_{i}}(I, G) - \pi^*_{I, G})
\end{equation}

\begin{theorem}
The heuristic described above is admissible and consistent for problem \ref{prob4}. 
\end{theorem}
The proof follows directly from Theorems \ref{thm2} and \ref{thm3}. 
Following the heuristics above, 
we can easily implement search algorithms for progressive explanations. 
For the motivating example, 
solving for Problems 1 and 2 would yield the desired explanation:

\vskip3pt
\noindent
{\tt
Amy: Let's go to the outlet today.\\
Monica: But today is a holiday (shops closed). \\
Amy: Too bad! Let us go to the central park then. \\
Monica: My car is ready. \\
Amy: Great! \\
Monica: And the rain will stop soon. \\
Amy: Excellent! \\
}
\vskip2pt
 
 An illustration of the changes of plan cost per explanation step for this example is in Fig. \ref{example}.
 Note that progressive explanation creates a smoother curve even though both explanations have the same number of steps,
 which is expected to contribute to a better understanding of the explanation. 
 \begin{figure}
\centering
  \includegraphics[width=1.0\columnwidth]{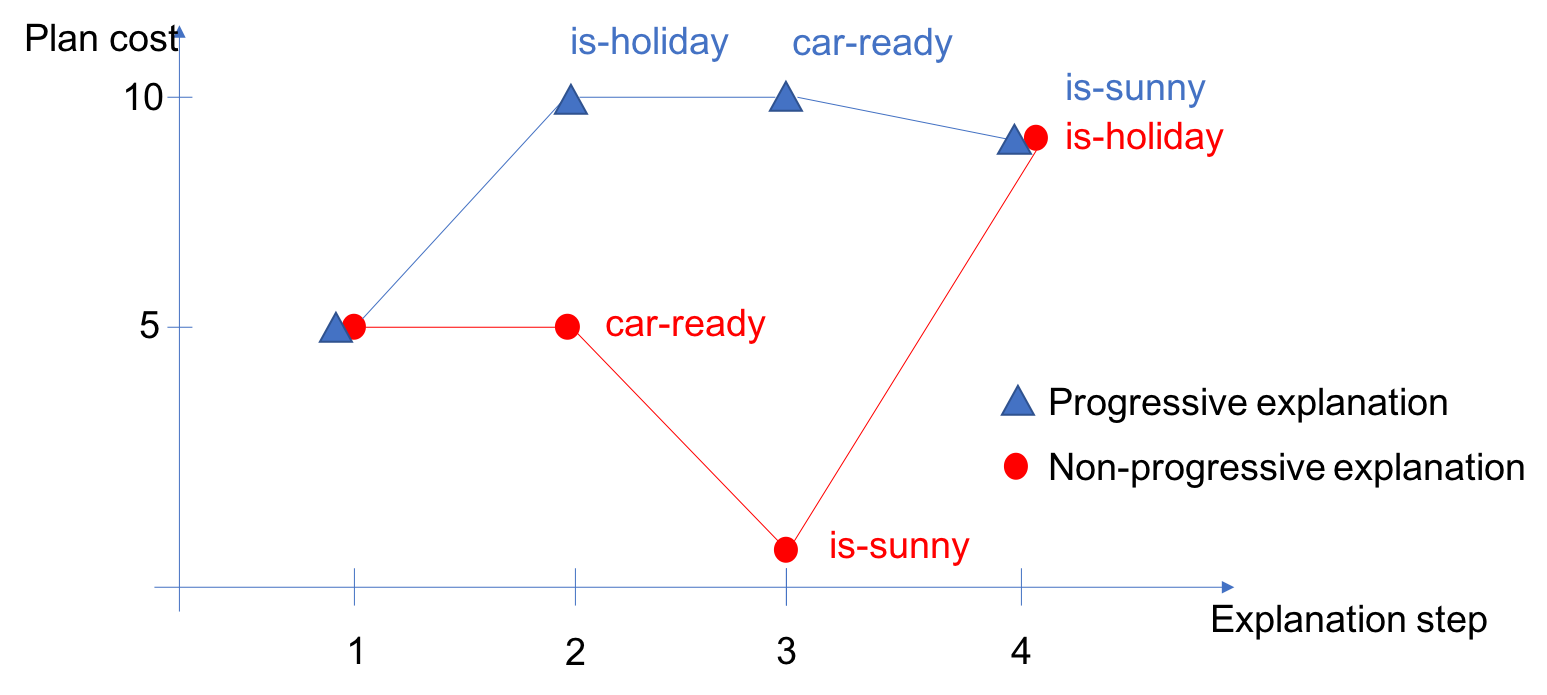}%
\caption{Changes of plan cost per explanation step for the motivating example. The curve of PEG is smoother.}
 \label{example}
\end{figure}

There are many other ways such model-plan distances may be defined. 
For example, one may prefer more significant changes to the model at the beginning than later in the explanation. 
Also, instead of plan editing distance, 
you may consider other common plan distance metrics, such as action,
state, and causal link distances \cite{fox2006plan}. 
Another interesting consideration is the influence of plan hierarchies \cite{erol1994htn, sreedharan2018hierarchical}. 
For example, one may consider aggregating similar feature changes into the same explanation step,
which introduce similar changes to the plan. 
The focus could also be on the changes to plan hierarchies.

\subsection{Planning Method}

Given the heuristics, the planning methods can be implemented as standard $A^*$ searches.  
At each step, the search algorithm can choose a unit feature change from all possible changes
that satisfy condition $1$ in Def. \ref{def:exp}.
As discussed, we may choose to first consider the changes that are more promising.
The search can easily incorporate other considerations such as conciseness. 
This is especially useful for cases when some feature changes do not affect the plans generated or their costs.
In such cases, progressive explanations may include those unnecessary changes.
This can be addressed by adding to the $g$ value a small cost per every change made. 

\section{Evaluation}
\label{sec:eval}

\begin{table*}[h]
\centering
\begin{tabular}{SScSrcSr} \toprule
    {Run Index} & {No. Missing Features} & {\textbf{PEG Size}} & {\textbf{PEG Time} (s)} & {\textbf{PEG} $\sum_i \rho_i$ (P2)} &  {Explanation Size} & {Time (s)} & {$\sum_i \rho_i$ (P2)}  \\ \midrule
    1  & 8 (75) 	& 2 & 150.0   & {\textcolor{blue}{10.0}}    & 2 & 38.3 & {\color{red}{16.0}}  \\
    2  & 8 (75)  	& 4 & 76.8  & {\textcolor{blue}{63.0}}  & 4 & 80.1 & {\textcolor{red}{169.0}}  \\
    3  & 6 (75)  	& 1 & 18.4     & 169.0  & 1 & 4.6 & 169.0  \\
    4  & 6 (75)  	& 3 & 27.6   & 14.0   & 3 & 16.9 & 14.0     \\ 
    5  & 5 (75)	& 3 & 11.6  & 169.0 & 3 & 9.6 & 169.0   \\
    6  & 10 (75)  	& 3 & 677.1& 169.0 & 3 & 75.6 &  169.0  \\
    7  & 8 (75)  	& 3 & 145.5  & {\textcolor{blue}{145.0}} & 3 & 36.9 &  {\textcolor{red}{169.0}}   \\
    8  & 9 (75)  	& {\textcolor{blue}{4}} & 325.9  & {\textcolor{blue}{87.0}} & {\textcolor{red}{3}} & 69.8  &  {\textcolor{red}{91.0}}    \\ 
    9  & 3 (75) 	& 0 & 0.5    & 0.0   & 0 & 0.5 & 0.0  \\
    10 & 9 (75)   & 1  & 167.8   & 4.0   & 1 & 5.0  &   4.0 \\ \bottomrule
    { Average} & {7.2 (9.6\%)}  & 2.4 & 160.0 & 83.0  & 2.3 & 33.7  & 97.0   \\ \bottomrule
\end{tabular}
 \caption{Performance comparison between generating progressive explanations and concise explanations with feature missing probability set to $0.1$. The result is based on rover problem \#1 in the IPC domain. }
 \label{table1}
\end{table*}

\begin{table*}[h]
\centering
\begin{tabular}{SScSrcSr} \toprule
    {Run Index} & {No. Missing Features} & {\textbf{PEG Size}} & {\textbf{PEG Time} (s)} & {\textbf{PEG} $\sum_i \rho_i$ (P2)} &  {Explanation Size} & {Time (s)} & {$\sum_i \rho_i$ (P2)}  \\ \midrule
    1  & 8 (75) 	& 2 & 213.4   & {\textcolor{blue}{2.0}}    & 2 & 38.2 & {\color{red}{4.0}}  \\
    2  & 8 (75)  	& 3 & 160.6  & 30.0  & 3 & 69.7 & 30.0  \\
    3  & 6 (75)  	& 1 & 18.4     & 64.0  & 1 & 4.5 & 64.0  \\
    4  & 6 (75)  	& 2 & 48.5   & 8.0   & 2 & 16.8 & 8.0     \\ 
    5  & 5 (75)	& 3 & 11.5  & 64.0 & 3 & 9.1 & 64.0   \\
    6  & 10 (75)  	& 3 & 666.7& 64.0 & 3 & 70.6 &  64.0  \\
    7  & 8 (75)  	& 2 & 128.0  & 64.0 & 2 & 30.4 & 64.0   \\
    8  & 9 (75)  	& 3 & 352.7  & {\textcolor{blue}{38.0}} & 3 & 64.0  &  {\textcolor{red}{40.0}}    \\ 
    9  & 3 (75) 	& 0 & 0.6    & 0.0   & 0 & 0.6 & 0.0  \\
    10 & 9 (75)   & 1  & 225.5   & 1.0   & 1 & 12.5  &   1.0 \\ \bottomrule
    { Average} & {7.2 (9.6\%)}  & 2.0 & 182.6 & 33.5  & 2.0 & 31.6  & 33.9   \\ \bottomrule
\end{tabular}
 \caption{Performance comparison between generating progressive explanations and concise explanations with feature missing probability set to $0.1$. The result is based on rover problem \#2 in the IPC domain. }
 \label{table2}
\end{table*}

For evaluation, we test our approach on the rover domain--a standard IPC domain. 
In this domain, the rover is to explore the space and communicate samples 
back to the base station via communication stations. 
The robot can sample rock and soil, as well as take images. 
To sample rock and soil, the rover must have an empty storage. 
Before taking an image of an objective, the rover must calibrate its camera
and the camera will need to be recalibrated before being used again. 
The evaluation is performed on a 2.8 GHz quad-core Macbook Pro computer with 16G memory.
The underlying planner is Fastdownward \cite{helmert2006fast}.
Also, since the search methods are similar except for the heuristics, 
we focus on evaluating the distance metric described in Problem \ref{prob2}.
For all the evaluations, we focus on differences between the action models
of $M^R$ and $M^H$,
meaning differences in action preconditions, add and delete effects.
The conclusion should naturally extend when other differences (e.g., differences in the initial state) are considered. 


Here, we compare progressive explanations with concise explanations. 
As we mentioned, a progressive explanation may not be the most concise explanation and vice versa. 
For searching progressive explanations, we also add a small cost to the $g$ value for every unit feature change. 
Given $M^R$, for each model feature in $\Gamma(M^R)$ that involves precondition, add and delete effects for actions, 
we associate it with a missing probability. 
For this evaluation, we set the probability to be $0.1$. 
We first create progressive explanations using $A^*$ with our heuristics.
We then run $A^*$ that only looks at the most concise explanation. 
For both cases, when a plan cannot be found in a model, the plan cost is considered to be $0$.
The result for one of the rover problems is presented in Table \ref{table1}.
The result for a second rover problem is  presented in Table \ref{table2}.
We can see that PEG explanations are almost as concise and at the same time
have a much lower $\sum_i \rho_i$ value in many cases,
although they generally take longer to be found.

Next, we test our algorithms as the number of missing features increases in $M^H$ with respect to $M^R$.
The result is presented in Table \ref{table3}.
As expected, in general, the time for the search increases as the number of missing features increases.
The size of an explanation seems to have little contribution to its computation time,
except when it is $0$ where almost no search work is required. 

\begin{table}[h]
\centering
\begin{tabular}{SSSS} \toprule
    {Missing Prob.} & {No. Missing Features} &  {\textbf{PEG Size}} & {\textbf{PEG Time} (s)}  \\ \midrule
    0.06  & 5 (75)  & 1 & 9.3  \\
    0.07  & 5 (75)  & 3 & 14.3   \\
    0.08  & 6 (75)  & 1 & 19.7   \\
    0.09  & 6 (75)  & 3 & 26.3      \\ 
    0.10  & 5 (75)   & 3 & 10.7    \\
    0.11  & 11 (75) & 3 & 1342.1    \\
    0.12  & 9 (75)  & 3 & 378.3     \\
    0.13  & 9 (75)  & 4 & 373.0       \\ 
    0.14  & 5 (75)  & 0 & 0.7   \\ \bottomrule   
\end{tabular}
 \caption{Time performance for generating PEG as the missing probability changes from $0.06$ to $0.14$ with a step size $0.01$. The result is based on rover problem \#1.}
  \label{table3}
\end{table}

\section{Conclusion}
\label{sec:con}

In this paper, we studied the problem of explanation generation. 
In contrast to prior work, we consider explanation generation
in a model reconciliation setting. 
The key consideration here is that explanation is meant for the explainee
and hence an explanation from the explainer's perspective may not be a desirable explanation.
We take a step further from our prior work by considering not only the right explanation for the explainee, 
but also the underlying cognitive effort required from the explainee for understanding the explanation, resulting in a general framework for PEG. 

An observation is that making explanation is an incremental process that 
constitutes of multiple steps. 
As a result, the cognitive effort can be computed as a sum of the cognitive effort required at each step. 
The goal then becomes minimizing the sum of such effort. 
This converts our explanation generation problem to a sequential decision making problem.
The cognitive effort at each step is associated with a model-plan distance metric.
Efficient search methods are provided for several distance metrics that are intuitively connected to cognition. 
Our approach is evaluated in a standard IPC domain.
Results comparing PEG and concise explanations show the correlation between the two types of explanations,
and illustrate the effectiveness of our search methods.

There are many possible future directions. 
The forefront is a thorough evaluation of progressive explanations with respect to other types of explanations.
Progressive explanations may also be used to introduce dialogs during the explaining process
by identifying when questions may be asked.

\vskip5pt
\noindent
\textbf{\textit{Acknowledgement:}} This research is supported in part by NSF Award 1844524.

\bibliography{bib.bib} 
\bibliographystyle{plain}

\end{document}